\documentclass{article}

\pdfoutput=1
\usepackage{arxiv}

\usepackage{graphicx}
\usepackage[colorlinks=true, allcolors=blue]{hyperref}
\usepackage{amsmath,amsfonts,amssymb,amsthm,latexsym}
\usepackage{algorithm}
\usepackage{algpseudocode}
\usepackage{smartdiagram}
\usepackage{lscape}
\usepackage{geometry}
\usepackage{booktabs}
\usepackage{rotating}
\usepackage{cellspace}
\usepackage{tabularray}
\UseTblrLibrary{diagbox}
\usepackage{diagbox}

\usepackage[utf8]{inputenc} 
\usepackage[T1]{fontenc}    
\usepackage{url}            
\usepackage{booktabs}       
\usepackage{amsfonts}       
\usepackage{nicefrac}       
\usepackage{microtype}      
\usepackage{doi}

\usepackage[english]{babel}
\usepackage{caption}
\usepackage{subcaption}
\usepackage{multirow}

\usepackage{setspace}

\newtheorem{definition}{Definition}
\newtheorem{prop}{Proposition}

\title{HySim: An Efficient Hybrid Similarity Measure for Patch Matching in Image Inpainting}

\author{
	Saad Noufel, Nadir Maaroufi, Mehdi Najib, Mohamed Bakhouya \\
	International University of Rabat \\
	College of Engineering and Architecture, TICLab, LERMA Lab \\
	Sala Al Jadida 11000, Morocco\\
	\texttt{\{saad.noufel, nadir.maaroufi, mehdi.najib, mohamed.bakhouya\}@uir.ac.ma} \\
}
\date{}

\renewcommand{\headeright}{}
\renewcommand{\undertitle}{}
\renewcommand{\shorttitle}{\textit{An Efficient Hybrid Similarity Measure for Patch Matching in Image Inpainting} }

\hypersetup{
pdftitle={A template for the arxiv style},
pdfsubject={q-bio.NC, q-bio.QM},
pdfauthor={David S.~Hippocampus, Elias D.~Striatum},
pdfkeywords={First keyword, Second keyword, More},
}

\begin{document}
\maketitle

\begin{abstract}
Inpainting, for filling missing image regions, is a crucial task in various applications, such as medical imaging and remote sensing. Trending data-driven approaches efficiency, for image inpainting, often requires extensive data preprocessing. In this sense, there is still a need for model-driven approaches in case of application constrained with data availability and quality, especially for those related for time series forecasting using image inpainting techniques. This paper proposes an improved model-driven approach relying on patch-based techniques. Our approach deviates from the standard Sum of Squared Differences (SSD) similarity measure by introducing a Hybrid Similarity (HySim), which combines both strengths of Chebychev and Minkowski distances. This hybridization enhances patch selection, leading to high-quality inpainting results with reduced mismatch errors. Experimental results proved the effectiveness of our approach against other model-driven techniques, such as diffusion or patch-based approaches, showcasing its effectiveness in achieving visually pleasing restorations.
\end{abstract}


\section{Introduction}
The human eye craves visual coherence, seeking meaning and wholeness even in fragments. This inherent drive fueled the remarkable advancements in computer vision, particularly in the realm of image inpainting. Like magic, this technology restores missing or damaged regions, breathes life back into faded photographs, and empowers artistic visions by removing unwanted elements. From enhancing medical scans for precise diagnoses to reconstructing precious historical documents, image inpainting's applications stretch across diverse fields. For instance, filling in missing data in medical scans to improve disease detection and treatment \cite{xu2023review}. Reconstructing corrupted areas in satellite images for accurate environmental monitoring \cite{Zhang_Siyu2018-ee}. Bringing old photographs back to life by repairing tears, scratches, and faded details \cite{hu2014fast}. Removing unwanted objects or manipulating images for creative expression \cite{yang2020generative}, \cite{shen2018deep}, \cite{zhang2022gan}, and \cite{lahiri2020prior}.

Yet, restoring visual integrity involves navigating two distinct approaches: model-driven and data-driven approaches \cite{li2017localization}, each with its own strengths and limitations.
Model-driven approaches, encompassing diffusion-based, and patch-based approaches, excel at handling small missing regions and local image structure. However, their brushstrokes falter when faced with extensive damage or intricate details, struggling to preserve semantic content and often falling prey to artificial repetitions, overlapping patches, and long processing times \cite{salem2021survey}.
This is where data-driven approaches strides onto the canvas. Leveraging the power of neural networks, these approaches capture intricate patterns and semantic information with stunning proficiency, offering a more adaptive solution to complex inpainting tasks. Yet, challenges remain in capturing long-range dependencies, maintaining global and local consistency, and avoiding unwanted artifacts in the inpainted regions \cite{xu2023review}. Additionally, training deep networks often requires vast amounts of diverse and high-quality training data, which can be a practical hurdle \cite{burlin2017deep}.

Our work focuses on examplar-based approaches, which is motivated by the proposed approach presented in \cite{criminisi2004region}. The procedure begins with a target region, representing the missing area, and a source region, constituted of the entire image except the target area. Subsequently, the approach utilizes exemplar-based technique to fill the target region with pixels from the source region. This process involves identifying a patch in the source region, which is similar to the one in the target region and transferring missing pixels from the source to the target region. The approach in \cite{criminisi2004region} has proven effective in image inpainting. However, its reliance on the sum of squared distances (SSD) for patch similarity assessment can introduce cumulative mismatch errors, potentially compromising inpainting quality, particularly in scenarios involving intricate textures and structures. To address this limitation, we propose a refined inpainting approach that builds upon \cite{criminisi2004region} framework but integrates a new similarity measure. This latter demonstrates enhanced sensitivity to structural and textural variations within image patches, leading to improved accuracy in patch matching. Crucially, this refinement is achieved without necessitating modifications to the core priority algorithm, preserving its computational efficiency.

In summary, the main contribution of the work presented in this paper are three folds:
\begin{itemize}
	\item Investigate the examplar-based approach using experimental testing.
	\item Shed more light on key differences between similarity and distance concepts by highlighting their inherent duality.
	\item Introduce a novel Hybrid Similarity (HySim) measure, which combines the strengths of both Minkowski and Chebyshev distances, leading to improved performance.
\end{itemize}

The reminder of this paper is structured as follow. Section 2 provides a comprehensive state-of-the-art review of the inpainting methods. Section 3 introduces our proposed hybrid similarity measure and its seamless integration into the inpainting process. Section 4 reveal our results showcasing the effectiveness of our approach compared to existing methods. Section 5 summarizes our findings, discussing the implications of our work, while highlighting potential future directions of this approach.

\section{Related work}
In the domain of image inpainting, three prominent approaches have emerged, each offering unique techniques to address the challenge of reconstructing missing or damaged regions in images. \textbf{\textit{Diffusion-based approaches}}, \textbf{\textit{Patch-based approaches}}, and \textbf{\textit{Deep Learning approaches}}, as depicted in Fig.\ref{fig:methods}.
\begin{figure}[H]
	\centering
	\includegraphics[width=\textwidth]{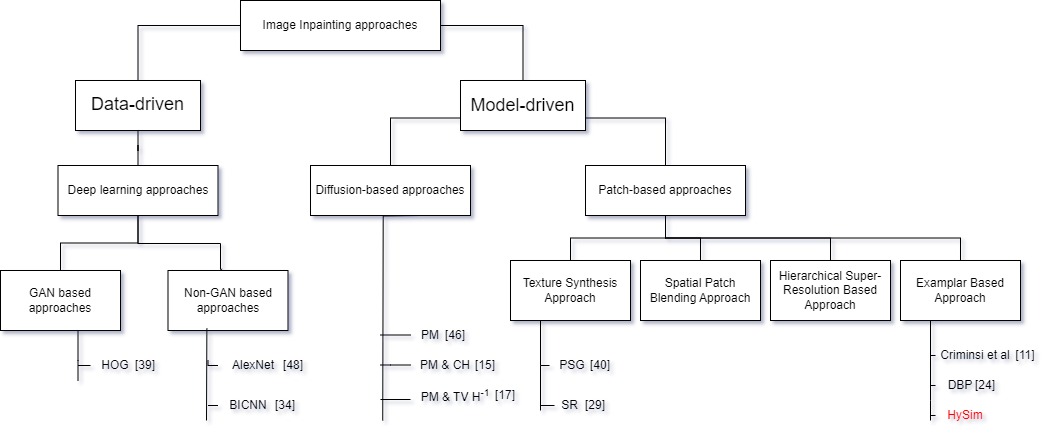}
	\caption{Comprehensive Overview of Image Inpainting Approaches.}
	\label{fig:methods}
\end{figure}

\subsection{Deep Learning Approach}
Leveraging the power of neural networks, deep learning approaches have demonstrated remarkable proficiency in capturing complex patterns and semantic content of images, offering a more adaptive and sophisticated solution to the challenges posed by inpainting tasks. The deep learning approaches can be classified into two main categories: \textit{Non-GAN} and \textit{GAN} based approaches.

The Non-GAN based approach uses Convolutional Neural Networks (CNN) to learn the patterns and features of an image and fill in missing regions based on that learning. For instance, the work presented in \cite{weerasekera2018just} incorporate the image's depth map as an input into the CNN architecture, while in \cite{zhao2018unsupervised}, they used CNN to inpaint X-ray medical images. In \cite{Chang2019-fz} another CNN-based approach specifically tailored for video inpainting was proposed, addressing the task of object removal. In contrast to conventional image inpainting techniques that require knowledge of the damaged pixels, \cite{cai2017blind} introduces a blind image inpainting approach known as BICNN. AlexNet \cite{krizhevsky2012imagenet}, based on CNN, uses multiple convolutional and pooling layers, followed by fully connected layers to extract meaningful features from the input image. The aim is to accurately predict missing parts of the image during the inpainting process. In \cite{zhu2018deep}, authors presented a patch-based inpainting method geared towards forensics images. However, CNN-based approaches often cannot effectively establish connections between missing areas and source known areas, often leading to structural distortions, texture blurring, and incoherence on the repaired area boundaries \cite{xu2023review}.

Meanwhile, the GAN-based approaches built upon generative adversarial network,  which is a framework that generates artificial samples indistinguishable from their real counterparts by competing subnetworks \cite{tang2017automatic}. For instance, the work in \cite{pathak2016context} presented the first GAN-based technique for image inpainting. They used context encoders to fill missing holes, based on the surrounding environment of the target area and the context of the whole image. On one hand, \cite{wang2022dual} proposed a pyramid context encoding network to inpaint images. It fills the damaged areas in a pyramid-style. On the other hand,  \cite{salem2021novel} proposed a face inpainting method that uses the Histogram of Oriented Gradient(HOG) as guidance to inpaint Humans in images. However,
the training of GAN is difficult since usually suffers from instability,
gradient disappearance and mode collapse \cite{zhang2023image}.

\subsection{Diffusion-based Approach}
Diffusion-based approach relies on Partial differential equations (PDEs), which offer a powerful toolkit for image denoising, effectively removing noise while preserving crucial image features and structures. Established techniques like anisotropic diffusion and spatial-fractional approaches demonstrate this effectiveness \cite{ashouri2022new}. In the field of image denoising, Perona and Malik \cite{perona1990scale} introduced a prominent feature-preserving algorithm known as the Perona-Malik (PM) model. This model utilizes an anisotropic diffusion equation, where a variable coefficient controls the diffusion process to effectively retain edges within the image.

\noindent Since diffusion-based approaches depends on PDE, the first work that proposed a PDE model for image inpainting was presented in \cite{bertalmio2000image}. It works by prolonging information from the surrounding areas to fill in the inpainted region (Fig.\ref{fig:diff}). However, this approach leads to blurry images and suffers from the presence of artifacts around the recovered region. In \cite{le2011examplar}, authors combines a PDE-based technique with exemplar-based, where a structure tensor for priority computing of patches is used for inpainting. The PDE approach propagates (i.e. diffuses), the local information of the image's surrounding region into the missing hole. \cite{sridevi2019image} presented a mathematical model, which is derived from using fractional-order derivative and Fourier transformation to remove text, noise, and blur in images. However, the main disadvantage of diffusion-based approaches is that it usually provides ambiguous result, especially when inpainting large missing holes. Diffusion-based approach is more suitable for completing simple areas of lines, curves and small missing areas.

\begin{figure}[H]
	\centering
	\includegraphics[width=.35\textwidth]{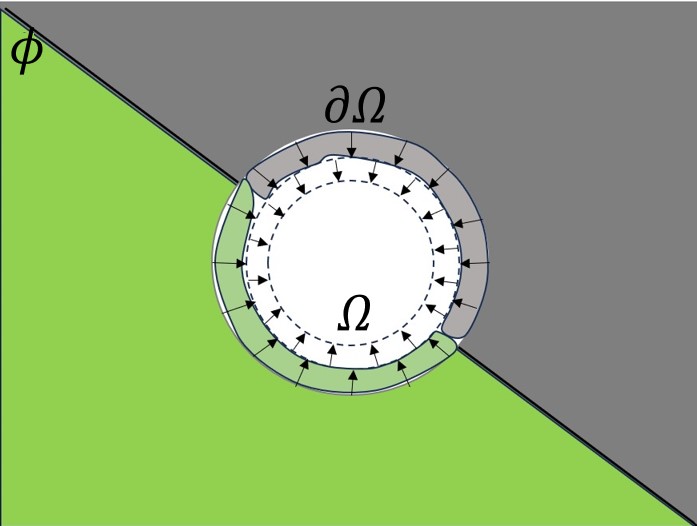}
	\caption{\textbf{Illustration of the diffusion-based inpainting process.} $\Phi$ is the source region, $\Omega$ is the target region, $\partial \Omega$ is $\Omega$ boundary.}
	\label{fig:diff}
\end{figure}

Further refining the diffusion process, the PM model offers a nuanced approach. This partial differential equation acts as a dynamic diffusion coefficient, carefully modulating the smoothing strength based on local image intensity gradients. Crucially, it reduces diffusion across strong edges, thereby preserving them while effectively smoothing noise and inconsistencies within uniform regions. Its selective behavior makes the PM model particularly well-suited for image inpainting, where reconstruction fidelity hinges on seamlessly filling missing areas without blurring crucial details.

Building upon this foundation, the work presented in \cite{zou2021image} proposed a fourth-order PDE model by ingeniously combining the PM equation with the Cahn-Hilliard (CH) equation. This approach stemmed from addressing the limitations of the Cahn-Hilliard equation, as observed in \cite{bertozzi2006inpainting}, where edge smoothness suffered. Recognizing the PM equation's effectiveness in smoothing, the authors sought to synergize the strengths of both equations. And so, they proposed a forth order PDE. While demonstrating superior results for grayscale images, the combined model exhibits limitations in achieving optimal smoothness while preserving linear structures. In \cite{banerjeeimage}, the authors proposed a new fourth order PDE-based model for image inpainting, which is based on total variation (TV - $H^{-1}$) coupled with PM Equation. However, despite their efficiency in handling inpainting of gray images, they are enable to handle complex textures and the tendency to produce over-smoothed results.

\subsection{Patch-based Approach}
Patch-based image inpainting is an approach where missing or damaged portions of an image are reconstructed by filling them with information from surrounding patches (Fig.\ref{fig:patch}). This approach is further classified into four sub-categories as shown in (Fig.\ref{fig:methods}). \textit{Texture Synthesis approach}, \textit{Spatial Patch Blending approach}, \textit{The hierarchical super-resolution approach}, \textit{Examplar-based approach}. Texture synthesis approach randomly uses near pixel to fill the missing information in the image. By introducing a notion of patching, these approaches select existing pixels in a similar neighborhood \cite{salem2021survey}. The work presented in \cite{casaca2014laplacian} proposed a new technique for efficiently filling in missing parts of textured images while preserving their texture and structure. It combines anisotropic diffusion, transport equation, and texture synthesis to achieve high-quality inpainting results. Even though this technique yielded positive outcomes, it has some limitations, such as parameters adjustment and unpleasant results when painting non-textured images with color variations.

\begin{figure}[H]
	\centering
	\includegraphics[width=.35\textwidth]{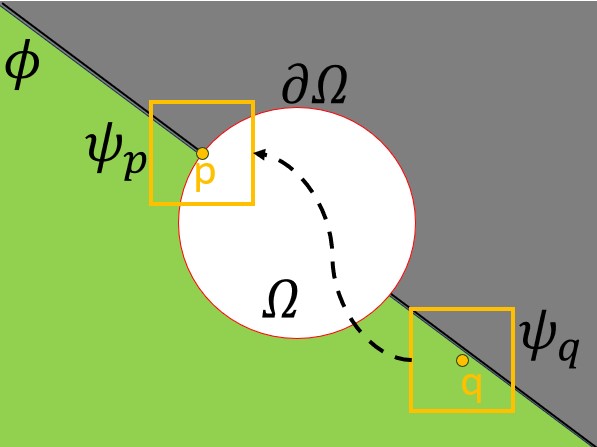}
	\caption{\textbf{Illustration of the patch-based inpainting process.} $\Phi$ is the source region, $\Omega$ is the target region, $\partial \Omega$ is $\Omega$ boundary, $\Psi_p$ is the target patch, $\Psi_q$ is the selected patch.}
	\label{fig:patch}
\end{figure}

\noindent Spatial Patch Blending approach \cite{daisy2013spatial} involves two key steps: artifact detection and spatial patch blending. In other patch-based inpainting approaches, isotropic blending is performed equally in all directions. However, spatial patch blending recognizes that discarded parts of individual patches contain valuable information and aims to reduce seams between patch pieces. This process uses a Gaussian weight function for blending, with computation time being a significant challenge dependent on mask size.

\noindent The hierarchical super-resolution based approach is a framework for inpainting, which helps recovering details on missing areas in low-resolution images, improving both computational complexity and visual quality \cite{le2013hierarchical}. The work in \cite{liu2023sequential} proposed Hierarchical super-resolution that explores hierarchical feature information using a sequential multi-scale block and distribution transformation block. Hierarchical super-resolution based approach can reconstruct high-resolution images from low-resolution images, but face challenges, such as low-pass filtering, translation, decimation, and noise corruption \cite{humblot2006super}.

\noindent Examplar-based inpainting approach pioneered by \cite{criminisi2004region}, laid the groundwork for subsequent advancements. Their seminal work introduced two key propositions: priority computation and patch selection, paving the way for further enhancements. Building upon this foundation, we present an overview of diverse exemplar-based approaches, critically analyzing their strengths and limitations.

\noindent For instance, state-of-the-art exemplar-based methods built upon \cite{criminisi2004region} have primarily focused on improving either priority computation, patch selection, or both. The first one aims to identify the most crucial regions to be filled, often based on criteria like information content, proximity to known areas, or edge structure. In contrast, the second one focuses on finding the most similar patches within the known image region to replace the missing parts.

\noindent In \cite{deng2015exemplar}, authors proposed two main contributions. Firstly, a new priority definition that encourages geometry propagation, and secondly, an algorithm to estimate steps of the new priority. While in the work proposed by \cite{zhang2020image}, authors addressed the issue of mismatch errors in image inpainting by introducing Differences Between Patches (DBP) measure in order to compare target and exemplar patches. Combining DBP with the standard SSD adaptively detects mismatch errors. Upon detection, a two-round search strategy with a new matching rule is applied to find a more suitable exemplar patch. Furthermore, \cite{fan2018novel} tackles the issue of inaccurate patch matching in image inpainting by proposing two key improvements, an adaptive patch approach that dynamically adjusts patch size based on local image sparsity, and a rotation-invariant matching approach that identifies the most similar exemplar patch regardless of its orientation. 

\noindent The work in \cite{ying2017improved} tackles image inpainting by first segmenting the image with a watershed algorithm, ensuring that texture information stays within regional boundaries. It then analyzes curvature features of isophotes to capture fine details and guides patch matching based on these features, leading to accurate texture reconstruction and minimal bleeding artifacts. The result is a simple approach that delivers high-quality inpainting. In \cite{nan2014improved}, authors proposed two improvements to the priority and confidence functions used in image inpainting. Firstly, it replaces the multiplication operation in the traditional priority function with addition and utilizes the golden section as the weight ratio. Secondly, the logistic equation, they have introduced, updates the confidence measure, preventing its rapid decay and allowing for more reliable confidence estimates throughout the inpainting process. These combined changes lead to improved accuracy and efficiency in image inpainting. 

\begin{figure}[H]
	\centering
	\includegraphics[width=.35\textwidth]{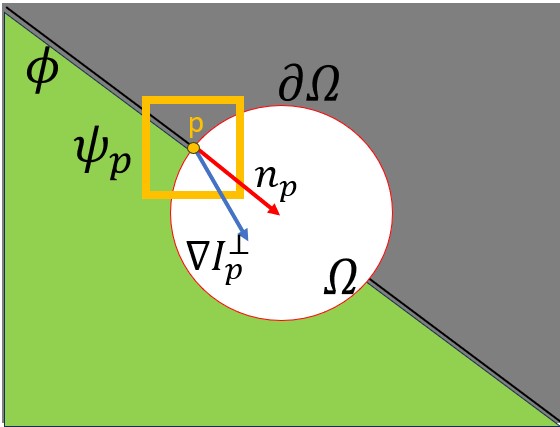}
	\caption{\textbf{Notation diagram.} $\Phi$ is the source region, $\Omega$ is the target region, $\partial \Omega$ is $\Omega$ boundary, $\Psi_p$ is the target patch, $n_p$ is the normal to the contour $\partial \Omega$, and $\nabla I_p^{\bot }$ is the isophote.}
	\label{fig:image}
\end{figure}

\noindent In \cite{xu2010image}, the authors introduce two concepts for patch-based image inpainting, focusing on patch priority and representation. Firstly, patch structure sparsity measures the confidence of a patch at image structures (edges, corners) based on the sparsity of its non-zero similarities with neighbors. Patches with higher sparsity receive higher priority for inpainting, prioritizing key structural elements. Secondly, the inpainted patch is modeled as a sparse linear combination of candidate patches under local consistency constraints. This sparse representation ensures sharp, texture-consistent results in the filled regions. Together, these concepts significantly enhance the effectiveness of patch propagation in examplar-based inpainting. The work in \cite{shen2009image} assumed that each image patch can be sparsely represented using a redundant dictionary. To ensure visual plausibility and consistency, the proposed approach constructs a dictionary by directly sampling from the intact regions of the image. This context-aware dictionary then guides the sequential recovery of each incomplete patch at the hole's boundary, progressively filling it until complete. The work in \cite{reshniak2020nonlocal}, proposes a unified nonlocal inpainting technique tackling structures and textures simultaneously. By combining multi-scale geometry recovery with nonlocal texture synthesis guided by an anisotropic metric, it achieves high-quality results, advancing image restoration.

\section{Examplar-based Approach : Overview and Limitation}
In order to shed more light on the efficiency of the examplar-based approach, we have conducted experiments by investigating mismatch error in the inpainting process.
As stated above, examplar-based approach follows two main steps: \textbf{Priority Computation}, and \textbf{Patch Selection from Source Region}.
The first step focuses on the priority computation of a patch. This latter is computed using a priority function, which determines the order in which patches are filled. This function is defined as $P(p) = C(p) \times D(p)$, where $C(p)$, the confidence term, measures the amount of reliable information surrounding the pixel p and $D(p)$, the data term, measures the difference between the gradient at the boundary of the target region and the gradient at the patch boundary. Their respective formulas computed in Eq(\ref{C(p)}) and Eq(\ref{D(p)}), where $\Psi_p$ is the target patch, $\Omega$ the source region, $n_p$ is the normal to the contour, $\nabla I_p^{\bot }$ is the isophote, and $\alpha$ is a normalizing term usually equal 255.
\begin{equation}
	C(p) = \frac{\sum_{q \in \Psi_p \cap \overline{\Omega}} C(q)}{\mid \Psi_p \mid}
	\label{C(p)}
\end{equation} 

\begin{equation}
	D(p) = \frac{\mid \nabla I_p^{\bot} \cdot n_p \mid}{\alpha}
	\label{D(p)}
\end{equation}
The priority function $P(p)$ thus combines the confidence and data terms to determine the order in which patches are filled, ensuring a balance between reliable information and preservation of image structures.

The second step concerns the patch selection from the source region. In fact, to find the best replacement for the target region, the patch selection searches for similar patches within the source region. This is achieved by comparing the target patch to potential matches in the source, using measures like texture and structure to determine the best fit. In \cite{criminisi2004region}, authors used the SSD metric to measure the best replacement patch out of all the possible patches (Eq.\ref{ssd}), where $\Psi_p$ and $\Psi_q$ are respectively the target and possible replacement patches, and $M$ the binary mask, where one indicates the pixels that needs filling and zero the ones already existing. However, SSD might replace the target patch by an inappropriate patch. The mismatch error keep getting accumulated along the inpainting process. This problem will definitely include undesired and unwanted objects in the desired target.

\begin{equation} 
	ssd(\Psi _{\textrm {p}},\Psi _{\textrm {q}})=\sum {\left |{ {\bar {M}\Psi _{\textrm {p}} -\bar {M}\Psi _{\textrm {q}}} }\right |^{2}}
	\label{ssd}
\end{equation}

\begin{figure}[H]
	\centering
	\begin{subfigure}[b]{.15\textwidth}
		\includegraphics[width=\textwidth]{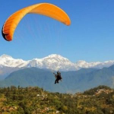}
		\caption{}
	\end{subfigure}
	\begin{subfigure}[b]{.15\textwidth}
		\includegraphics[width=\textwidth]{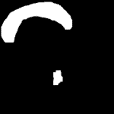}
		\caption{}
		\label{(b)}
	\end{subfigure}
	\begin{subfigure}[b]{.15\textwidth}
		\includegraphics[width=\textwidth]{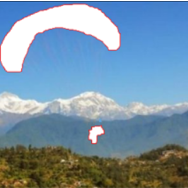}
		\caption{}
		\label{(c)}
	\end{subfigure}
	\begin{subfigure}[b]{.15\textwidth}
		\includegraphics[width=\textwidth]{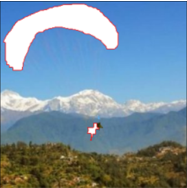}
		\caption{}
		\label{(d)}
	\end{subfigure}
	\begin{subfigure}[b]{.15\textwidth}
		\includegraphics[width=\textwidth]{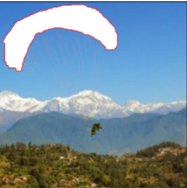}
		\caption{}
		\label{(e)}
	\end{subfigure}
	\begin{subfigure}[b]{.15\textwidth}
		\includegraphics[width=\textwidth]{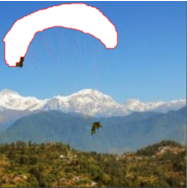}
		\caption{}
		\label{(f)}
	\end{subfigure}
	\begin{subfigure}[b]{.15\textwidth}
		\includegraphics[width=\textwidth]{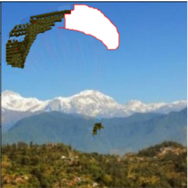}
		\caption{}
		\label{(g)}
	\end{subfigure}
	\begin{subfigure}[b]{.15\textwidth}
		\includegraphics[width=\textwidth]{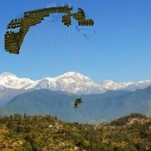}
		\caption{}
		\label{(h)}
	\end{subfigure}
	\caption{A visual proof of mismatch error. (a) the original image, (b) the target region, which is marked in white, (c) the first iteration of \cite{criminisi2004region}, (d)-(g) Different iteration of the inpainting process, (h) the last iteration.}
	\label{fig:lkl}
\end{figure}

Related experimental result as depicted in Fig.\ref{fig:lkl}, we show an instance of obtained results, since all of them have shown the same behavior. The main objective of this experiment was to remove the parachutist with his parachute, while maintaining a homogeneous background with a blue sky in the top, and a mountain structure in the middle of the image. As expected, after selecting the target region in image \ref{(b)} , the priority function prioritizes the parachutist section. Image \ref{(c)} showcases that the first iteration are going smooth, and the replacement selected looks reasonable. However, image \ref{(d)} showcases the introduction of a bad replacement patch, a small green patch taken from the forest was chosen as a replacement instead of a mountain structure. This error keeps getting propagated and accumulated in images \ref{(e)} - \ref{(g)}. The final result in image \ref{(h)} left us with a poor inpainted image full of unwanted objects. 

It is worth nothing that from this experimental result, the above mentionned mismatch issue does not come from the first step of the approach (i.e., Priority computation),  it is most likely related to the patch selection, due to the inadequate filling, as depicted in image \ref{(h)}. In fact, in this case study, the problem comes from applying SSD similarity measure, since it tends to choose the worst replacement out of all existing patches. 

This work builds upon understanding that selecting the most similar patch for filling becomes the critical factor. The pioneering work \cite{criminisi2004region} employed the SSD metric to assess patch similarity. However, Fig.\ref{fig:lkl} serves as visual evidence supporting our hypothesis that SSD can lead to error accumulation, and thus incoherent inpainted images, even in this simple case study. Recent work presented in \cite{zhang2020image} stated the same conclusion and proposed an enhanced SSD similarity measure to overcome this issue. However, the authors  did not modify the SSD core but, instead, they proposed a combination of SSD with DBP, and an additional conditional constraint.

In this research work, we have introduced another similarity measure (HySim) for efficiently selecting most similar patches with the main aim is to tackle the issue of error accumulations. The basics and principles of the proposed similarity approach is presented in the following section.

\section{The proposed approach}
In this section, we introduce the HySim approach, but it is essential, in the context of image processing, to first understand the relationship between patches' measure of similarities and distances. Despite the wide utilization of these concepts, to the best of our knowledge, a clear definition of similarity measure concept is highly required. In fact, this lack of clarity often leads sometimes to confusion between the concept of similarity and distance. So, before introducing our HySim approach, the next subsection is completely dedicated to some related definitions and notations, while highlighting the relationship between similarity and distance measure. 
\subsection{Definitions and notations}
In the framework of patch-based image inpainting, we have reformulated the definition of similarity measure, presented in \cite{huang2020novel}. There exist another definition of similarity presented in \cite{chen2009similarity}, with more constraint than the on defined here.

\begin{definition}
	A similarity measure is a function $s$ that maps pairs of patches,  included in $\Phi \times \Phi$ ($\Phi^2$), to a non-negative real number ($\mathbb{R}^+$).
	Given two patches,  denoted $\Psi_p$ and $\Psi_q$, of same and fixed size in the image domain $\Phi$. The function $s$ must verify the following straightforward criteria (SSM):
	\begin{enumerate}
		\item \textbf{Symmetry (S):} the similarity between $\Psi_p$ and $\Psi_q$ is the same as the similarity between $\Psi_q$ and $\Psi_p$:
		\begin{equation}
			s(\Psi_p, \Psi_q) = s(\Psi_q, \Psi_p), \forall \;(\Psi_q, \Psi_p)\; \in\; \Phi^2
		\end{equation}
		\item \textbf{Self-resemblance Maximization (SM):} the similarity between any two distinct patches, $\Psi_p$ and $\Psi_q$,  is always less than the maximum similarity $S \in \mathbb{R}^{\ast,+}$ of the function $s$, which is reached for $p=q$:
		\begin{equation}
			s(\Psi_p, \Psi_q) < S = s(\Psi_p, \Psi_p) ,\;\forall (\Psi_q, \Psi_p) \in \Phi^2
		\end{equation}
	\end{enumerate}
\end{definition}

\begin{prop}
	Any distance can be used to construct a similarity measure.
\end{prop}

\begin{proof}
	Let us consider a given distance $d$, which verifies the well known axioms (symmetry, identity of indiscernible, and triangle inequality). Let $S$ = $\max d(\Psi_p, \Psi_q)$ the maximum distance between two distinct patches $\Psi_p$ and $\Psi_q$. The function $s$ defined by $s(\Psi_p, \Psi_q) = S - d(\Psi_p, \Psi_q)$ is a similarity measure. According to the above-mentioned definition, a similarity measure must verify \textit{SSM} criteria. Let us verify that $s$ is a similarity measure as follows:
	\begin{enumerate}
		\item \textbf{(S):} $s(\Psi_p, \Psi_q) = S - d(\Psi_p, \Psi_q) = S - d(\Psi_q, \Psi_p) = s(\Psi_q, \Psi_p) \; \forall(\Psi_p, \Psi_q)\; \in\; \Phi^2$
		\item \textbf{(SM):} $s(\Psi_p, \Psi_q) = S - d(\Psi_p, \Psi_q) < S -\underbrace{d(\Psi_p, \Psi_p)}_{= 0}  =  s(\Psi_p, \Psi_p)\;\forall(\Psi_p, \Psi_q)\; \in\; \Phi^2$
	\end{enumerate}
\end{proof}
To summarize, there is a duality between similarity and distance, i.e., maximizing the similarity between two patches inherently translates to minimizing the distance between them.

\subsection{HySim}
Motivated by the above mentioned results, it is clear that the set of distances is included in the set of similarities. Therefore, we investigated the set of distances to find a distance that would solve the mismatch problem. After performing multiple tests on various types of images, we found two distances of interest. The first is the Minkowski distance, which focuses on a local feature scale. When $P=1$, it is called the Manhattan distance (Fig.\ref{manhat}). On the other hand, as the parameter $P$ in the Minkowski distance formula approaches infinity, it converges towards the Chebyshev distance (Fig.\ref{cheb}). The Chebyshev distance focuses on a global scale. Drawing inspiration from the strengths of both the Minkowski and Chebyshev distances, we propose the HySim approach - a hybrid similarity measure (Eq.\ref{Cheb-Mink}) for patch selection in exemplar-based inpainting. This similarity measure focuses both locally and globally.
\begin{equation}
	d_{\alpha, \beta}^{P, \infty} = \alpha\; \max \mid \Psi_p - \Psi_q \mid + \beta\; \sqrt[P]{\sum \mid \Psi_p - \Psi_q \mid ^P}  
	\label{Cheb-Mink}
\end{equation}

\begin{figure}[H]
	\centering
	\begin{subfigure}[b]{.15\textwidth}
		\includegraphics[width=\textwidth]{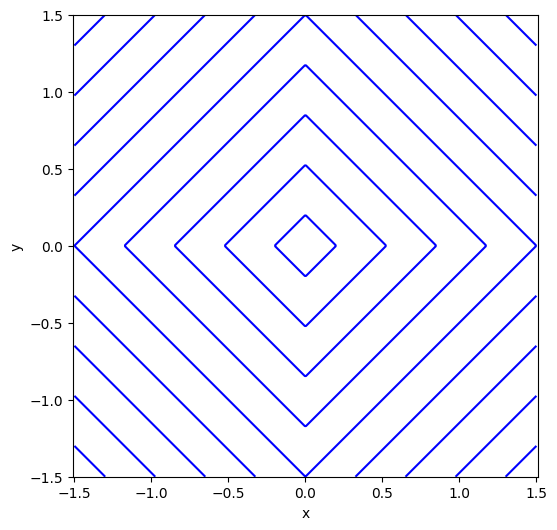}
		\caption{}
		\label{manhat}
	\end{subfigure}
	\begin{subfigure}[b]{.15\textwidth}
		\includegraphics[width=\textwidth]{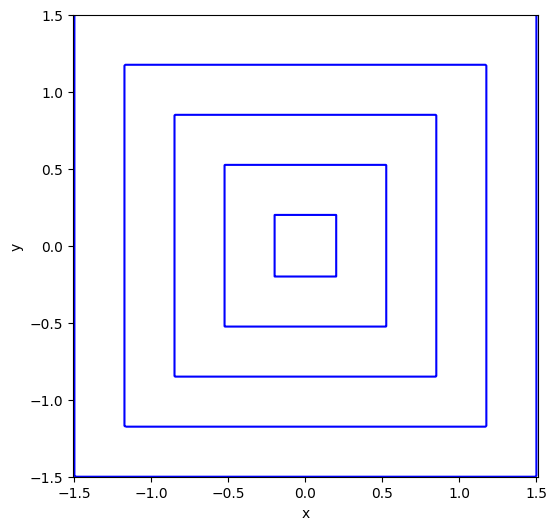}
		\caption{}
		\label{cheb}
	\end{subfigure}
	\begin{subfigure}[b]{.15\textwidth}
		\includegraphics[width=\textwidth]{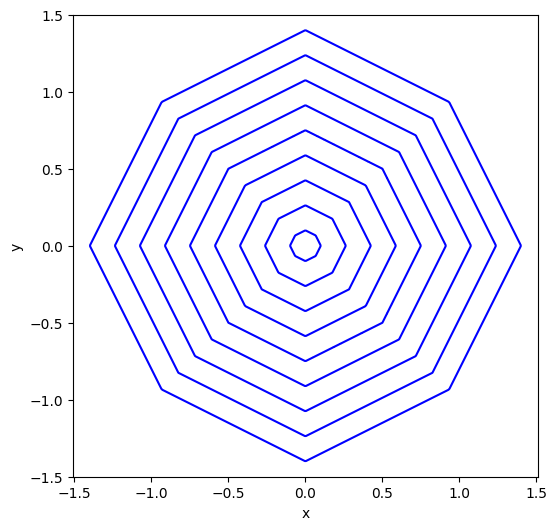}
		\caption{}
		\label{cheb_mink}
	\end{subfigure}
	\begin{subfigure}[b]{.15\textwidth}
		\includegraphics[width=\textwidth]{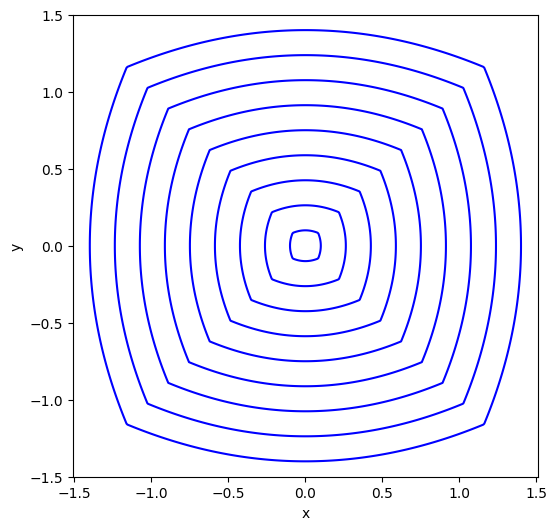}
		\caption{}
		\label{cheb_mink_2}
	\end{subfigure}
	\begin{subfigure}[b]{.15\textwidth}
		\includegraphics[width=\textwidth]{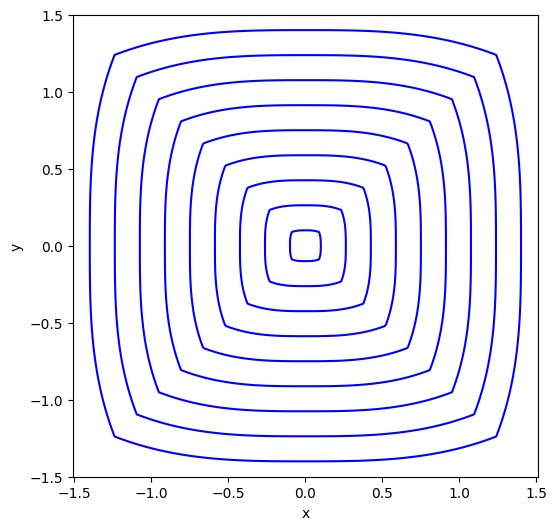}
		\caption{}
		\label{cheb_mink_3}
	\end{subfigure}
	\begin{subfigure}[b]{.15\textwidth}
		\includegraphics[width=\textwidth]{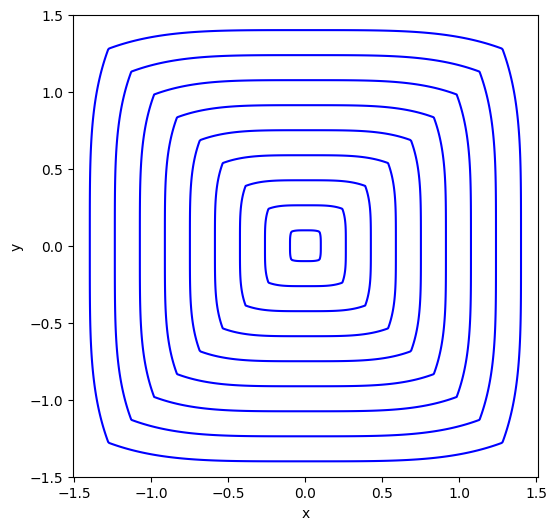}
		\caption{}
		\label{cheb_mink_4}
	\end{subfigure}
	\caption{Graphical representation in $\mathbb{R}^2$ of the Chebchev, Minkowski and their combination distance measure. (a) the Manhattan distance, (b) the Chebychev distance, (c) HySim approach for $P=1$, (d) HySim approach for $P=2$, (e) HySim approach for $P=3$, (f) HySim approach for $P=4$.}
	\label{fig:boule}
\end{figure}

The approach in question is fundamentally a measure of distance, with its validity as such established by \cite{rodrigues2018combining}, demonstrating that it meets the criteria for distance measures when $P>1$, they also highlighted the possibility of using it in image processing. As illustrated in Fig.\ref{fig:boule}, increasing the value of $P$ brings the measure closer to the Chebyshev distance. Therefore, it is advisable to maintain a lower value for the power to optimize its time complexity as also stated in \cite{rodrigues2018combining} and to minimize its mismatch error.
HySim is designed to achieve high accuracy in patch matching by leveraging both the strengths of Minkowski and Chebychev distance. And, by testing on the image in Fig.\ref{fig:lkl}, our proposed approach resulted in a smooth, well inpainted, image Fig.\ref{(i)}. The selected patches assured the completion of the sky and the mountain of trees.

\begin{figure}[h]
	\centering
	\begin{subfigure}[b]{.15\textwidth}
		\includegraphics[width=\textwidth]{./inc/_a_}
		\caption{}
		\label{(a)2}
	\end{subfigure}
	\begin{subfigure}[b]{.15\textwidth}
		\includegraphics[width=\textwidth]{./inc/_h_}
		\caption{}
		\label{(h)2}
	\end{subfigure}
	\begin{subfigure}[b]{.15\textwidth}
		\includegraphics[width=\textwidth]{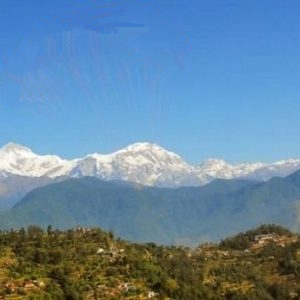}
		\caption{}
		\label{(i)}
	\end{subfigure}
	\caption{A visual comparison of our approach against \cite{criminisi2004region}. (a) the original image with the parachutist in the sky. (b) the output of \cite{criminisi2004region}. (c) our approach inpainted image.}
\end{figure}

\noindent To carry out our proposed approach, we formulated the following algorithm (\ref{alg:cap}). This latter,  systematically addresses the challenges of patch replacement within the exemplar-based inpainting framework. It incorporates the HySim similarity measure and leverages its advantages for better patch selection.

\begin{algorithm}[H]
	\caption{The proposed Algorithm}
	\label{alg:cap}
	\begin{algorithmic}
		\Require Original Image and a corresponding mask.
		\Ensure Restored image.
		\While{There existe a pixel in M equal to one}
		\State \textcolor{red}{/$\ast$ \textbf{Step 1:} Priority computation}
		\For{pixel p $\in$ $\partial \Omega$}
		\State $P(p)\;=\;C(p)\times D(p)$ 
		\EndFor
		\State \textcolor{red}{/$\ast$ \textbf{Step 2:} Selected Patch}
		\State $\Psi _{\textrm {p}} =\textrm{max}(P(\textrm {p}))$ 
		\State \textcolor{red}{/$\ast$ \textbf{Step 3:} Patch replacement search}
		\For{patch $\Psi_q$ in $\Phi$}
		\State $d_{\alpha, \beta}^{P, \infty} = \alpha\; \max \mid \Psi_p - \Psi_q \mid + \beta\; \sqrt[P]{\sum \mid \Psi_p - \Psi_q \mid ^P}$ 
		\EndFor
		\State  \textcolor{red}{/$\ast$ \textbf{Step 4:} Replacement selection}
		\State $\Psi _{\hat{q}} =\min (d_{\alpha, \beta}^{P, \infty})$ 
		\State  \textcolor{red}{/$\ast$ \textbf{Step 5:} Image replacing}
		\State $M \Psi_p = M \Psi _{\hat{q}}$
		\EndWhile
	\end{algorithmic}
\end{algorithm}

\section{Experimental Setup}
This section presents the experimental setup employed to evaluate the effectiveness of our proposed HySim approach. On one hand, we rigorously tested the performance of HySim against established techniques with public available code. In the other hand, we tested and compared our approach against images used in \cite{zhang2020image}. This controlled environment allows for a comprehensive comparison and objective assessment of HySim's capabilities. The evaluation process involved testing against available algorithms, applying them to the test images, and qualitatively analyzing the resulting inpainted images.
\subsection{Setup and Datasets}
To demonstrate the effectiveness of our proposed approach, we conducted multiple comprehensive evaluations on various image datasets (COCO \cite{lin2014microsoft} or the BSDS \cite{MartinFTM01}). We commenced by testing our HySim algorithm on fundamental geometric shapes to establish a baseline performance. Subsequently, we expanded the evaluation to encompass images with diverse textures, gradually increasing the difficulty level from easy to complex textures. We conducted the experiment with the $\alpha = 1$, $\beta = 1$, and $P \in \{1,2,3,4\}$ All experiments were meticulously executed on a computer equipped with an AMD Ryzen 7 processor (2.30 GHz) and 12 GB of RAM to ensure consistent and controlled testing conditions.
\subsection{Results and Discussion}
After testing HySim's capability in handling the mismatch error problem. We then showcased its effectiveness on various challenges, such as contour recognition. This section delves into the results and their implications, offering a deeper understanding HySim's capabilities.
\subsubsection{Basic Geometric Shapes}
To establish a baseline performance and assess the fundamental capabilities of HySim, we commenced our evaluation with images containing basic geometric shapes. This initial test will provide insights into the approach's ability to handle fundamental image features before exploring more complex scenarios.

We present three basic geometric shaped images in Fig.\ref{fig:green}-\ref{fig:curve}. Fig.\ref{fig:green} showcases HySim's ability to handle a common inpainting scenario - removing a medium-sized green dot from a background with two distinct colors, black and gray. The key challenge here is to seamlessly blend the restored area with the surrounding texture while maintaining the clear linear distinction between the black and gray regions. Ideally, HySim should reconstruct the missing area such that the black and gray areas remain intact, with no visible color bleeding, no mismatch patches, while keeping the separation between black and gray as linear as possible. This test provides valuable insights into the algorithm's effectiveness in preserving existing image details during the inpainting process. 

In the other hand, Fig.\ref{fig:triangle} challenges HySim's ability to reconstruct missing areas while strictly preserving the original geometric structure. Here, the top portion of a green triangle has been removed from a white background. The key challenge for HySim lies in accurately replicating the missing triangular section, ensuring the reconstructed edge seamlessly connects with the existing green edges while maintaining the triangle's overall shape. This test specifically assesses HySim's capability in handling sharp geometric features and in maintaining precise boundaries during the inpainting process. A successful restoration of this image demonstrated HySim's ability to accurately reconstruct missing regions while preserving the original image geometry. 

Fig.\ref{fig:curve} displays HySim's ability to handle the intricate task of restoring a missing element within a curved line, specifically focusing on maintaining both line smoothness and thickness. Here, a small red disk has been removed from the middle of a medium-thick curved line. The challenge lies in reconstructing the missing disk while ensuring the restored line segment seamlessly integrates with the existing sections. Unlike other approaches, HySim achieves a significant improvement by maintaining a slight curvature at the point of removal, even though it is not perfectly smooth. This demonstrates HySim's capability to not only reconstruct the missing element but also consider the original curvature of the line to a greater extent.

\begin{figure}[H]
	\centering
	\begin{subfigure}[b]{.15\textwidth}
		\includegraphics[width=\textwidth]{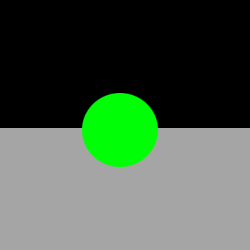}
		\caption{}
	\end{subfigure}
	\begin{subfigure}[b]{.15\textwidth}
		\includegraphics[width=\textwidth]{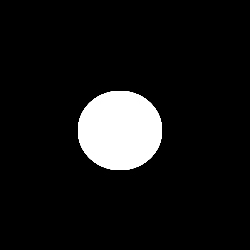}
		\caption{}
	\end{subfigure}
	\begin{subfigure}[b]{.15\textwidth}
		\includegraphics[width=\textwidth]{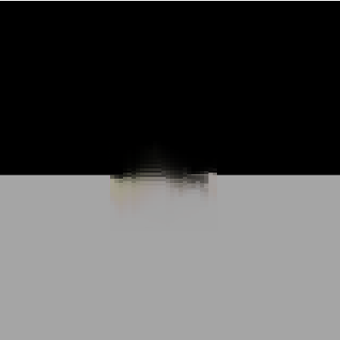}
		\caption{}
	\end{subfigure}
	
	\begin{subfigure}[b]{.15\textwidth}
		\includegraphics[width=\textwidth]{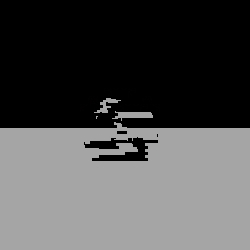}
		\caption{}
	\end{subfigure}
	\begin{subfigure}[b]{.15\textwidth}
		\includegraphics[width=\textwidth]{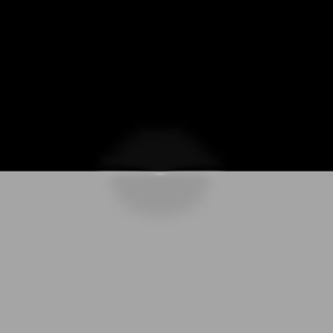}
		\caption{}
	\end{subfigure}
	\begin{subfigure}[b]{.15\textwidth}
		\includegraphics[width=\textwidth]{./inc/image7.jpg}
		\caption{}
	\end{subfigure}
	\begin{subfigure}[b]{.15\textwidth}
		\includegraphics[width=\textwidth]{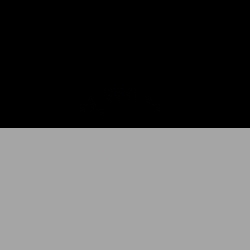}
		\caption{}
	\end{subfigure}
	\caption{Restoration results of first image. (a) the original image, (b) the target region, which is marked in white, (c) the result of approach in \cite{shen2009image}, (d) the result of approach in \cite{criminisi2004region}, (e) the result of approach in \cite{huang2014image}, (f) the result of the Perona-Malik approach, (g) the result of proposed approach, HySim.}
	\label{fig:green}
\end{figure}

\begin{figure}[H]
	\centering
	\begin{subfigure}[b]{.15\textwidth}
		\includegraphics[width=\textwidth]{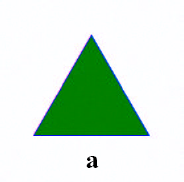}
		\caption{}
	\end{subfigure}
	\begin{subfigure}[b]{.15\textwidth}
		\includegraphics[width=\textwidth]{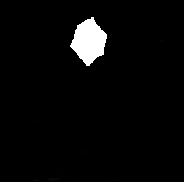}
		\caption{}
	\end{subfigure}
	\begin{subfigure}[b]{.15\textwidth}
		\includegraphics[width=\textwidth]{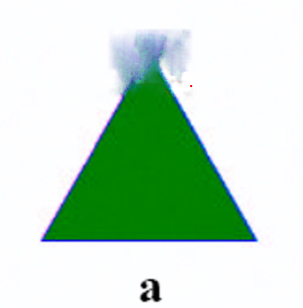}
		\caption{}
	\end{subfigure}
	\begin{subfigure}[b]{.15\textwidth}
		\includegraphics[width=\textwidth]{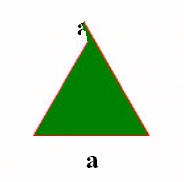}
		\caption{}
	\end{subfigure}
	\begin{subfigure}[b]{.15\textwidth}
		\includegraphics[width=\textwidth]{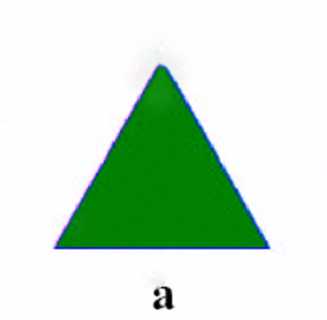}
		\caption{}
	\end{subfigure}
	\begin{subfigure}[b]{.15\textwidth}
		\includegraphics[width=\textwidth]{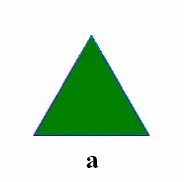}
		\caption{}
	\end{subfigure}
	\begin{subfigure}[b]{.15\textwidth}
		\includegraphics[width=\textwidth]{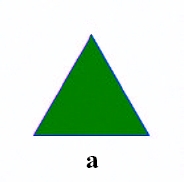}
		\caption{}
	\end{subfigure}
	\caption{Restoration results of second image. (a) the original image, (b) the target region, which is marked in white, (c) the result of approach in \cite{shen2009image}, (d) the result of approach in \cite{criminisi2004region}, (e) the result of approach in \cite{huang2014image}, (f) the result of the Perona-Malik approach, (g) the result of proposed approach, HySim.}
	\label{fig:triangle}
\end{figure}

\begin{figure}[H]
	\centering
	\begin{subfigure}[b]{.15\textwidth}
		\includegraphics[width=\textwidth]{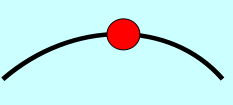}
		\caption{}
	\end{subfigure}
	\begin{subfigure}[b]{.15\textwidth}
		\includegraphics[width=\textwidth]{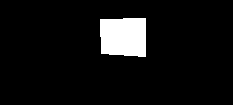}
		\caption{}
	\end{subfigure}
	\begin{subfigure}[b]{.15\textwidth}
		\includegraphics[width=\textwidth]{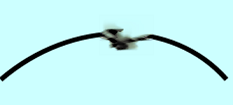}
		\caption{}
	\end{subfigure}
	\begin{subfigure}[b]{.15\textwidth}
		\includegraphics[width=\textwidth]{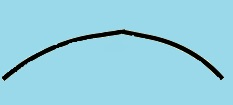}
		\caption{}
	\end{subfigure}
	\begin{subfigure}[b]{.15\textwidth}
		\includegraphics[width=\textwidth]{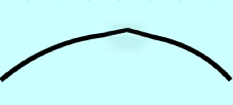}
		\caption{}
	\end{subfigure}
	\begin{subfigure}[b]{.15\textwidth}
		\includegraphics[width=\textwidth]{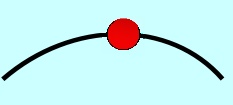}
		\caption{}
	\end{subfigure}
	\begin{subfigure}[b]{.15\textwidth}
		\includegraphics[width=\textwidth]{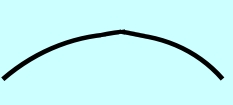}
		\caption{}
	\end{subfigure}
	\caption{Restoration results of third image. (a) the original image, (b) the target region, which is marked in white, (c) the result of approach in \cite{shen2009image}, (d) the result of approach in \cite{criminisi2004region}, (e) the result of approach in \cite{huang2014image}, (f) the result of the Perona-Malik approach, (g) the result of proposed approach, HySim.}
	\label{fig:curve}
\end{figure}
\subsubsection{Texture-rich Images}
Having established HySim's effectiveness on basic geometric shapes, we further evaluated our approach's restoration performance on texture-rich images, showcased in Fig.\ref{fig:a}-\ref{fig:c}. Each figure presents a progressively more complex scenario. For instance, Fig.\ref{fig:a} showcases the removal of a bear from a lake bathed in bright sunlight. The bear, partially submerged in the image's center, leaves a significant void requiring seamless restoration. Here, HySim faces the challenge of accurately replicating water ripples caused by the bear and the interplay of sunlight reflections. Additionally, since the bear occupies the center, HySim needs to seamlessly fill the gap while maintaining overall scene coherence, ensuring consistent lighting and smooth integration of the restored water with the surrounding undisturbed areas. This scenario showcases HySim's ability to handle complex natural environments and intricate details in real-world image restoration. 

Building on the complexities of the last figure, Fig.\ref{fig:b} displays HySim's ability to handle complex real-world scenarios by presenting a challenging image of a baseball game in progress. The image depicts a player swinging their bat in the middle ground, requiring HySim to seamlessly remove them while preserving the intricate background textures (grass, dirt around the mound). This demands cautious handling of diverse textures, object details, and dynamic poses, ensuring the reconstructed area blends flawlessly with the remaining scene. In the other hand, Fig. \ref{fig:c} showcases the removal of a wooden bench from a smooth concrete ledge. This scenario demands advanced capabilities from HySim as it needs to accurately choose replacement patches without interfering with the stone underneath the bench. HySim tackles this complexity by reconstructing the missing area while maintaining consistency in textures and seamlessly integrating the restored ledge surface around the existing stone, demonstrating its ability to handle intricate object-scene relationships and account for additional elements in real-world image restoration tasks.

\begin{figure}[H]
	\centering
	\begin{subfigure}[b]{.15\textwidth}
		\centering
		\includegraphics[width=\textwidth]{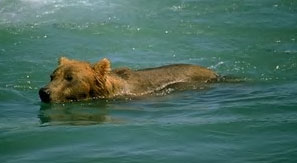}
		\caption{}
	\end{subfigure}
	\begin{subfigure}[b]{.15\textwidth}
		\centering
		\includegraphics[width=\textwidth]{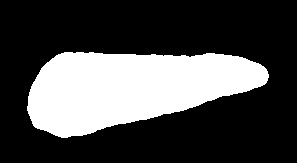}
		\caption{}
	\end{subfigure}
	\begin{subfigure}[b]{.15\textwidth}
		\centering
		\includegraphics[width=\textwidth]{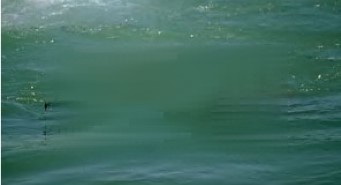}
		\caption{}
	\end{subfigure}
	\begin{subfigure}[b]{.15\textwidth}
		\centering
		\includegraphics[width=\textwidth]{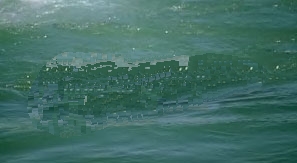}
		\caption{}
	\end{subfigure}
	\begin{subfigure}[b]{.15\textwidth}
		\centering
		\includegraphics[width=\textwidth]{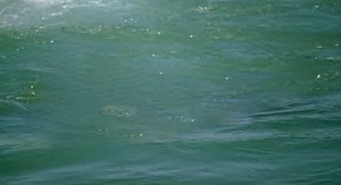}
		\caption{}
	\end{subfigure}
	\begin{subfigure}[b]{.15\textwidth}
		\centering
		\includegraphics[width=\textwidth]{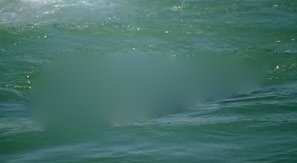}
		\caption{}
	\end{subfigure}
	\begin{subfigure}[b]{.15\textwidth}
		\centering
		\includegraphics[width=\textwidth]{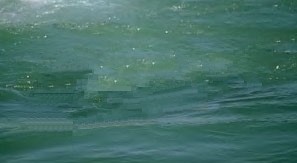}
		\caption{}
	\end{subfigure}
	\caption{Restoration results of "bear" image. (a) the original image, (b) the target region, which is marked in white, (c) the result of approach in \cite{shen2009image}, (d) the result of approach in \cite{criminisi2004region}, (e) the result of approach in \cite{huang2014image}, (f) the result of the Perona-Malik approach, (g) the result of proposed approach, HySim.}
	\label{fig:a}
\end{figure}

\begin{figure}[h]
	\centering
	\begin{subfigure}[b]{.15\textwidth}
		\centering
		\includegraphics[width=\textwidth]{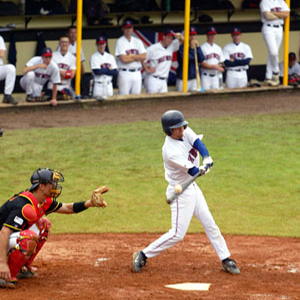}
		\caption{}
	\end{subfigure}
	\begin{subfigure}[b]{.15\textwidth}
		\centering
		\includegraphics[width=\textwidth]{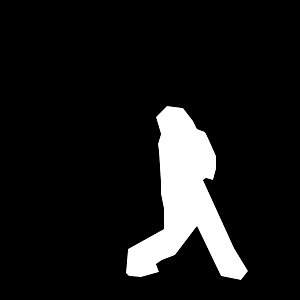}
		\caption{}
	\end{subfigure}
	\begin{subfigure}[b]{.15\textwidth}
		\centering
		\includegraphics[width=\textwidth]{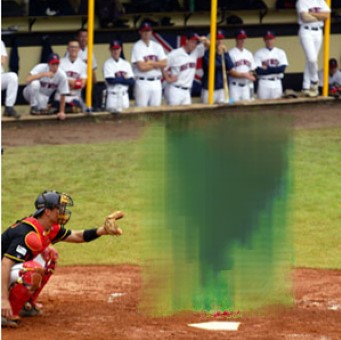}
		\caption{}
	\end{subfigure}	
	\begin{subfigure}[b]{.15\textwidth}
		\centering
		\includegraphics[width=\textwidth]{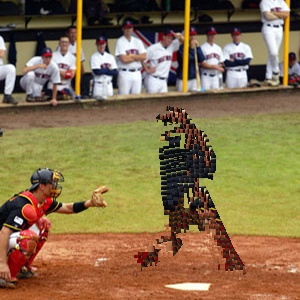}
		\caption{}
	\end{subfigure}
	\begin{subfigure}[b]{.15\textwidth}
		\centering
		\includegraphics[width=\textwidth]{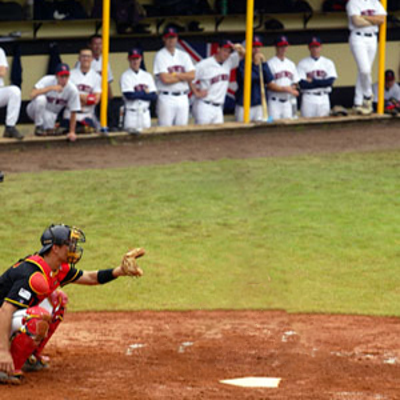}
		\caption{}
	\end{subfigure}
	\begin{subfigure}[b]{.15\textwidth}
		\centering
		\includegraphics[width=\textwidth]{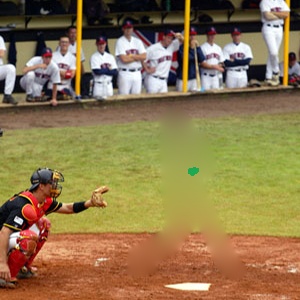}
		\caption{}
	\end{subfigure}
	\begin{subfigure}[b]{.15\textwidth}
		\centering
		\includegraphics[width=\textwidth]{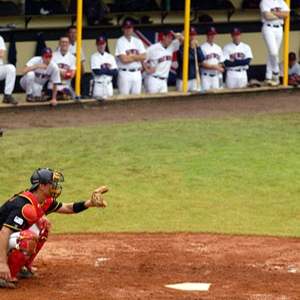}
		\caption{}
	\end{subfigure}
	\caption{Restoration results of "player" image. (a) the original image, (b) the target region, which is marked in white, (c) the result of approach in \cite{shen2009image}, (d) the result of approach in \cite{criminisi2004region}, (e) the result of approach in \cite{huang2014image}, (f) the result of the Perona-Malik approach, (g) the result of proposed approach, HySim.}
	\label{fig:b}
\end{figure}

\begin{figure}[H]
	\centering
	\begin{subfigure}[b]{.15\textwidth}
		\centering
		\includegraphics[width=\textwidth]{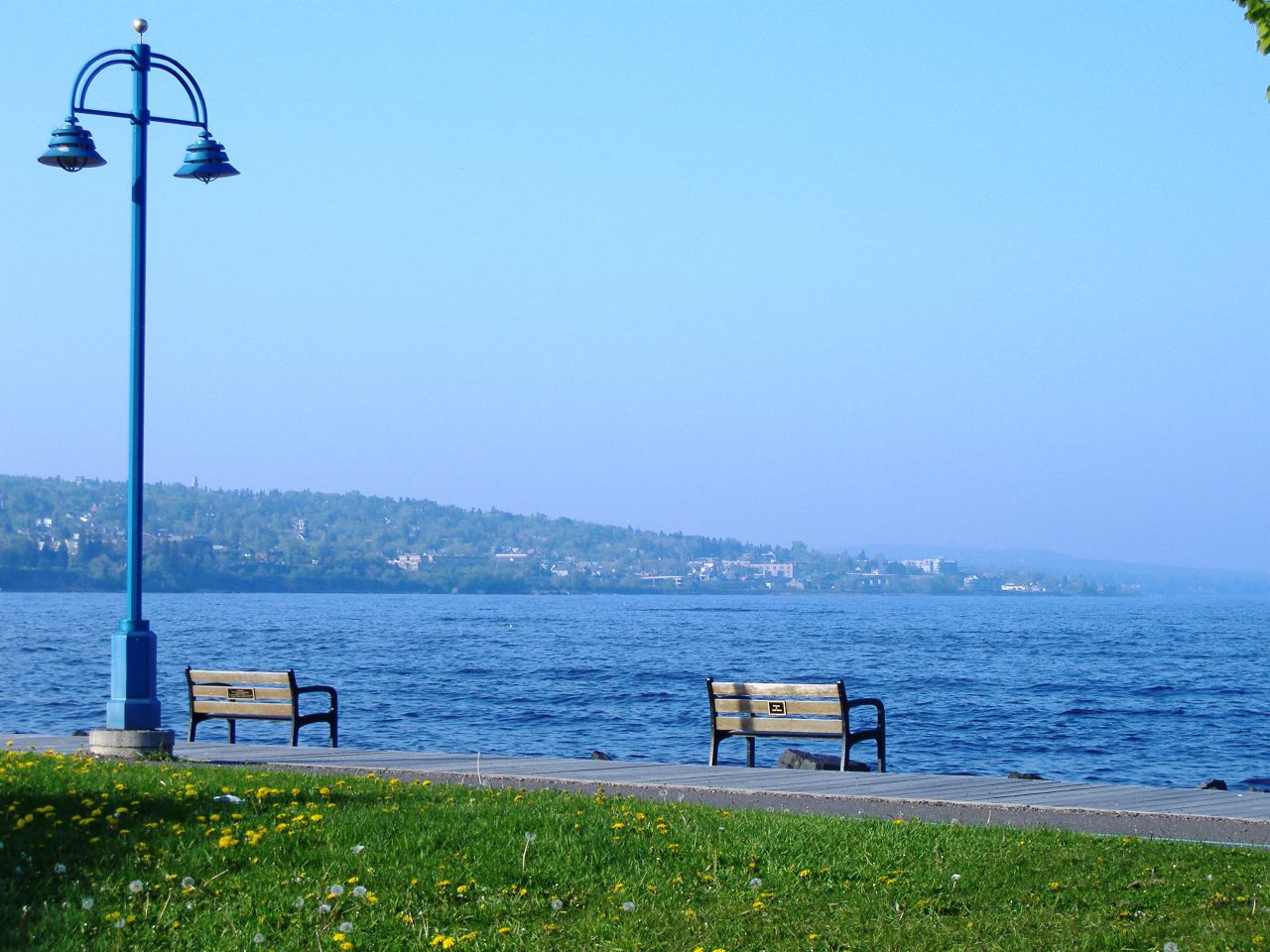}
		\caption{}
	\end{subfigure}
	\begin{subfigure}[b]{.15\textwidth}
		\centering
		\includegraphics[width=\textwidth]{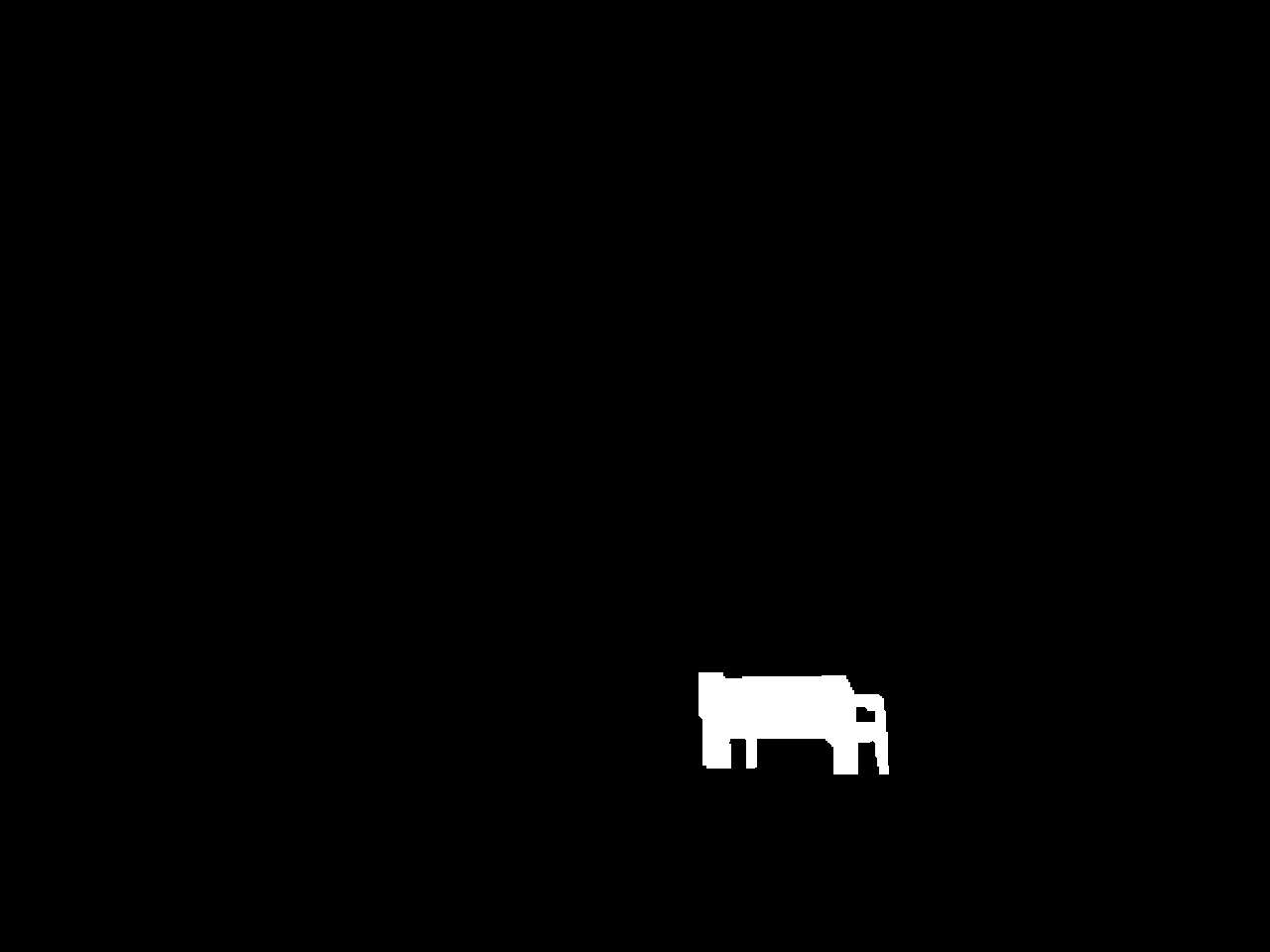}
		\caption{}
	\end{subfigure}
	\begin{subfigure}[b]{.15\textwidth}
		\centering
		\includegraphics[width=\textwidth]{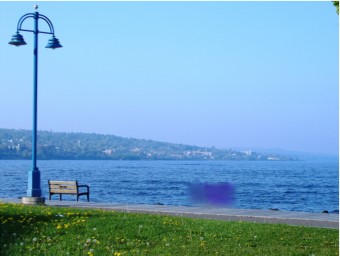}
		\caption{}
	\end{subfigure}
	\begin{subfigure}[b]{.15\textwidth}
		\centering
		\includegraphics[width=\textwidth]{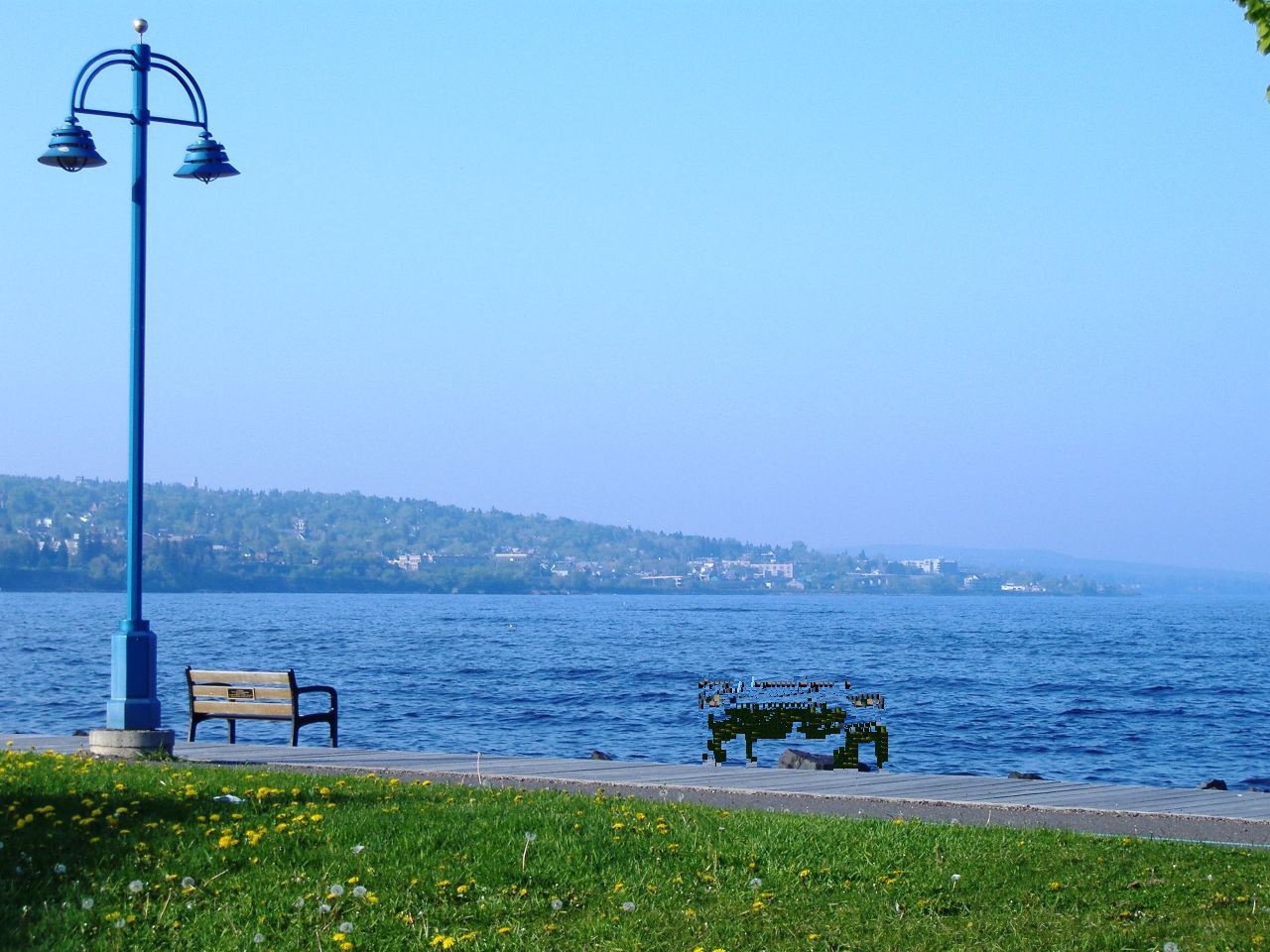}
		\caption{}
	\end{subfigure}
	\begin{subfigure}[b]{.15\textwidth}
		\centering
		\includegraphics[width=\textwidth]{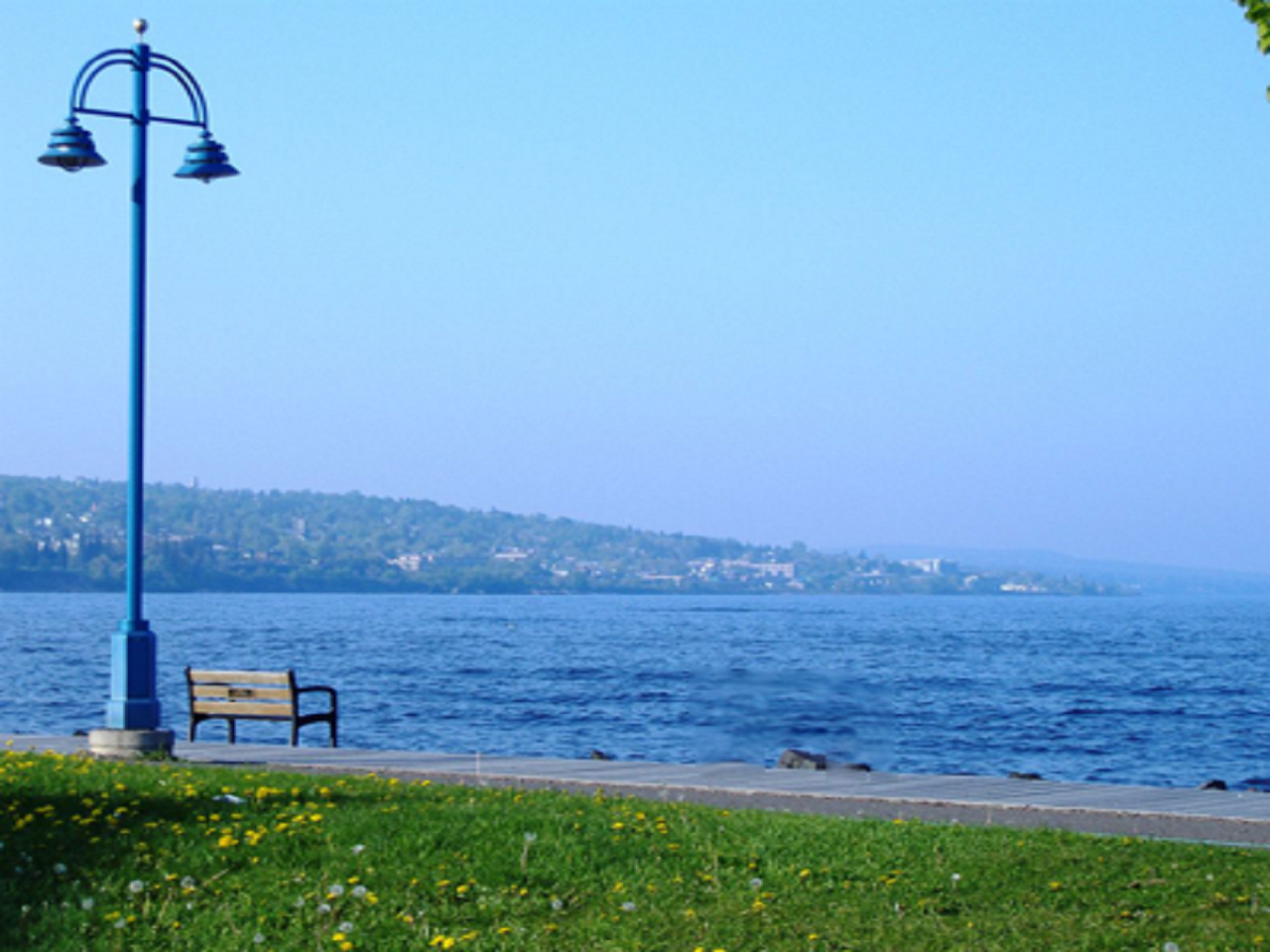}
		\caption{}
	\end{subfigure}
	\begin{subfigure}[b]{.15\textwidth}
		\centering
		\includegraphics[width=\textwidth]{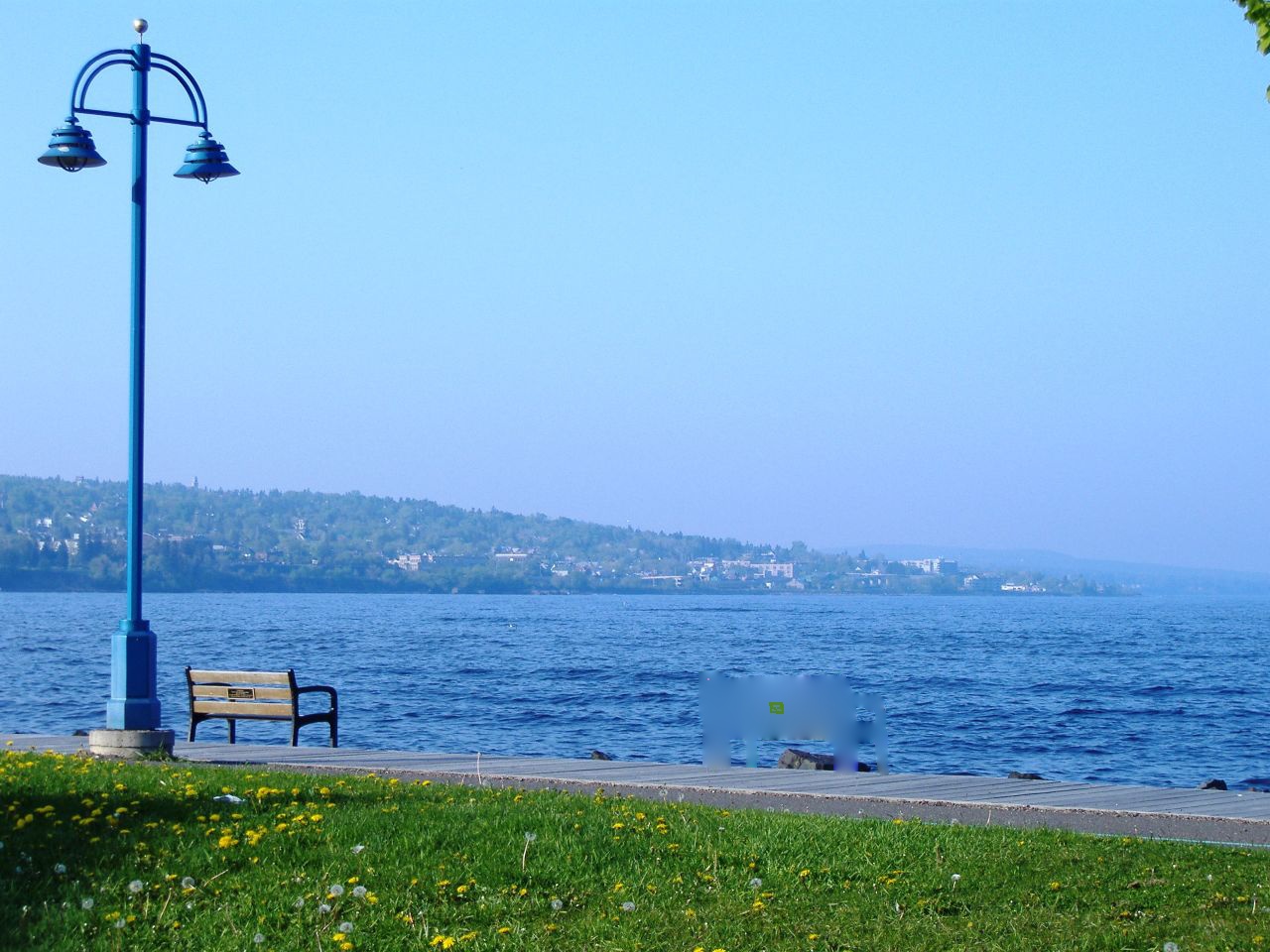}
		\caption{}
	\end{subfigure}
	\begin{subfigure}[b]{.15\textwidth}
		\centering
		\includegraphics[width=\textwidth]{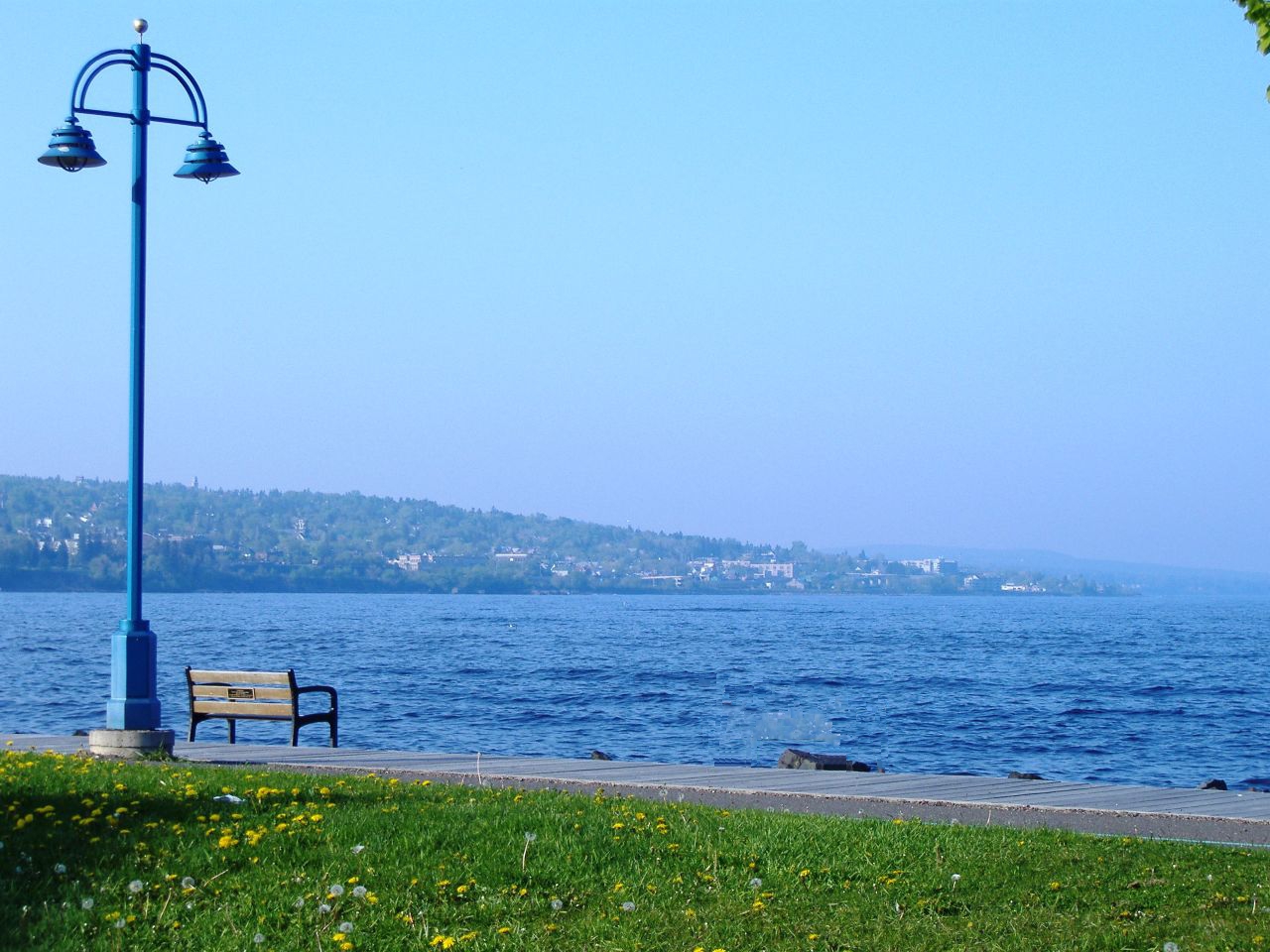}
		\caption{}
	\end{subfigure}
	\caption{Restoration results of "bench" image. (a) the original image, (b) the target region, which is marked in white, (c) the result of approach in \cite{shen2009image}, (d) the result of approach in \cite{criminisi2004region}, (e) the result of approach in \cite{huang2014image}, (f) the result of the Perona-Malik approach, (g) the result of proposed approach, HySim.}
	\label{fig:c}
\end{figure}

In order to further demonstrate the effectiveness of our approach, we compared it with several images from the previous work presented in \cite{zhang2020image}. Our testing showed that our approach produced visually superior inpainted results. In particular, HySim was more successful in avoiding line discontinuity and minimizing mismatch errors. 

Fig.\ref{fig:cyclist} showcases the capability of the HySim approach to handle objects with clear, straight lines, even in complex situations. Our approach successfully removed the cyclist entirely, including areas originally obscured by his shadow and his back. Remarkably, the filled-in region blends seamlessly with the existing straight red line (previously hidden by the shadow) and the blue straight line partially covered by the cyclist. This demonstrates the HySim approach's effectiveness in reconstructing these straight lines while maintaining a visually consistent image, resulting in a more natural and aesthetically pleasing restoration.

\noindent On the other hand, Fig.\ref{fig:hill} showcases another strength of our inpainting approach: removing large objects while avoiding mismatch error. In this scenario, the challenge involved removing a large, leafy tree including its shadow without disrupting the grassy hill behind it and the lake visible in the distance. Here, successful inpainting relies on accurately replicating the complex patterns of the grass on the hill. As evident in the figure, our approach seamlessly blended the inpainted area with the existing landscape using only green pixels, effectively replicating the grassy texture. Notably, the approach in \cite{zhang2020image} introduced a slight visual inconsistency by including blue pixels from the lake into the inpainted region, resulting in a mismatch error.
\begin{figure}[H]
	\centering
	\begin{subfigure}[b]{.15\textwidth}
		\centering
		\includegraphics[width=\textwidth]{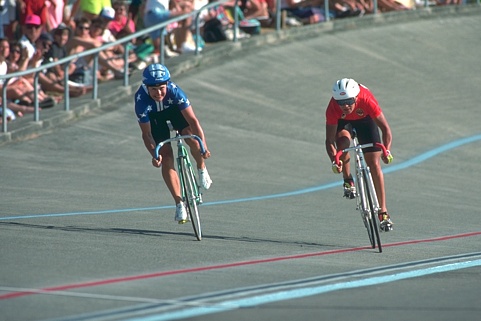}
		\caption{}
	\end{subfigure}
	\begin{subfigure}[b]{.15\textwidth}
		\centering
		\includegraphics[width=\textwidth]{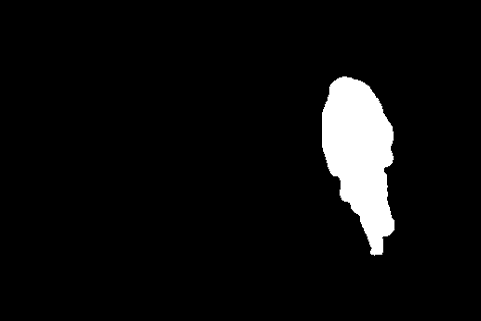}
		\caption{}
	\end{subfigure}
	\begin{subfigure}[b]{.15\textwidth}
		\centering
		\includegraphics[width=\textwidth]{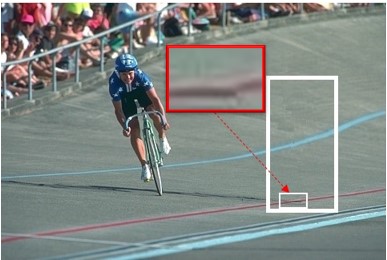}
		\caption{}
	\end{subfigure}
	\begin{subfigure}[b]{.15\textwidth}
		\centering
		\includegraphics[width=\textwidth]{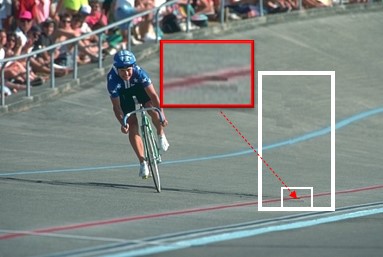}
		\caption{}
	\end{subfigure}
	\caption{Restoration results of the "cyclist" image. (a) the original image, (b) the target region, which is marked in white, (c) the result of approach in \cite{zhang2020image}, (d) the proposed approach.}
	\label{fig:cyclist}
\end{figure}

\begin{figure}[H]
	\centering
	\begin{subfigure}[b]{.15\textwidth}
		\centering
		\includegraphics[width=\textwidth]{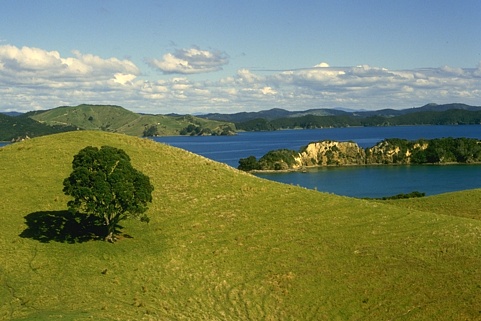}
		\caption{}
	\end{subfigure}
	\begin{subfigure}[b]{.15\textwidth}
		\centering
		\includegraphics[width=\textwidth]{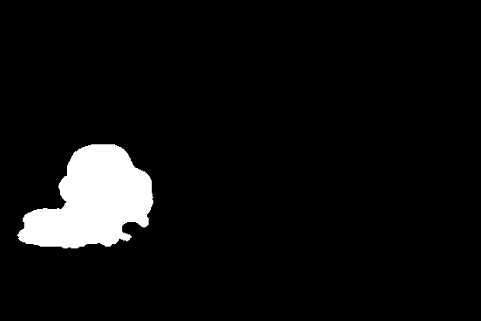}
		\caption{}
	\end{subfigure}
	\begin{subfigure}[b]{.15\textwidth}
		\centering
		\includegraphics[width=\textwidth]{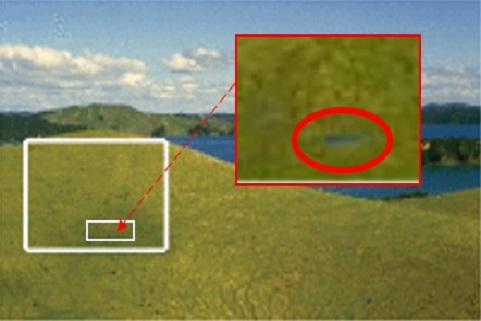}
		\caption{}
	\end{subfigure}
	\begin{subfigure}[b]{.15\textwidth}
		\centering
		\includegraphics[width=\textwidth]{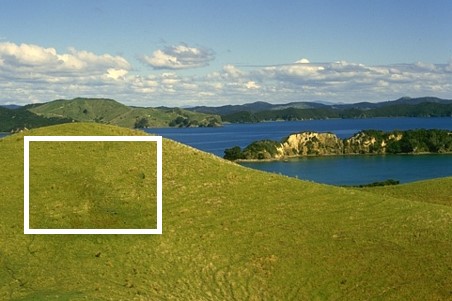}
		\caption{}
	\end{subfigure}
	\caption{Restoration results of the "hill" image. (a) the original image, (b) the target region, which is marked in white, (c) the result of approach in \cite{zhang2020image}, (d) the proposed approach.}
	\label{fig:hill}
\end{figure}

\section{Conclusions and Perspectives}
In this paper, we have first reformulated a definition of similarity measures in patch-based image inpainting, while highlighting that a similarity can be constructed from a distance. We then introduced HySim by combining Chebychev and Minkowski distances. We have conducted multiple experiments using different types of images, ranging from basic geometric shape to texture-rich images. Results have been presented and shown that our approach outperformed various model-driven approaches in avoiding mismatch error and respecting the shape contour. In fact, using HySim to enhance patch selection led to high-quality inpainting results (i.e., efficient restorations) with reduced mismatch errors. Ongoing work further investigates the inherent relationship between high-quality image inpainting and high-accuracy time series forecasts \cite{maaroufi2021predicting}. Preliminary results showcased the existing strong relationship between them, but still more experiments have to be conducted to show the effectiveness of high quality inpainting images on forecasting accuracy.

\end{document}